\documentclass{article}

\usepackage{microtype}
\usepackage{graphicx}
\usepackage{booktabs} 

\usepackage[utf8]{inputenc} 
\usepackage[T1]{fontenc}    
\usepackage{url}            
\usepackage{amsfonts}       
\usepackage{nicefrac}       
\usepackage{amssymb}
\usepackage{amsmath}
\usepackage{amsthm}
\usepackage{mathtools}
\usepackage{bbm}
\usepackage{fancyhdr}
\usepackage{adjustbox}
\usepackage{multirow}
\usepackage{float}
\usepackage{natbib}
\usepackage{xcolor}
\usepackage{wrapfig}
\usepackage{stmaryrd}
\usepackage{mathrsfs}
\usepackage[normalem]{ulem}
\usepackage{caption}
\usepackage{subcaption}
\usepackage{dsfont}

\usepackage{./notations}

\usepackage{hyperref}

\usepackage[accepted]{icml2020}

\icmltitlerunning{Regularized Optimal Transport is Ground Cost Adversarial}

\begin{document}

\twocolumn[
\icmltitle{Regularized Optimal Transport is Ground Cost Adversarial}



\icmlsetsymbol{equal}{*}

\begin{icmlauthorlist}
\icmlauthor{Fran\c{c}ois-Pierre Paty}{ensae}
\icmlauthor{Marco Cuturi}{google,ensae}
\end{icmlauthorlist}

\icmlaffiliation{ensae}{CREST / ENSAE Paris, Institut Polytechnique de Paris}
\icmlaffiliation{google}{Google Brain}

\icmlcorrespondingauthor{Fran\c{c}ois-Pierre Paty}{francois.pierre.paty@ensae.fr}

\icmlkeywords{Optimal Transport}

\vskip 0.3in
]



\printAffiliationsAndNotice{}  

\begin{abstract}
    Regularizing the optimal transport (OT) problem has proven crucial for OT theory to impact the field of machine learning. For instance, it is known that regularizing OT problems with entropy leads to faster computations and better differentiation using the Sinkhorn algorithm, as well as better sample complexity bounds than classic OT.
    In this work we depart from this practical perspective and propose a new interpretation of regularization as a robust mechanism, and show using Fenchel duality that any convex regularization of OT can be interpreted as ground cost adversarial. This incidentally gives access to a robust dissimilarity measure on the ground space, which can in turn be used in other applications. We propose algorithms to compute this robust cost, and illustrate the interest of this approach empirically.
\end{abstract}

\section{Introduction}

Optimal transport (OT) has become a generic tool in machine learning, with applications in various domains such as supervised machine learning~\cite{FrognerNIPS,abadeh2015distributionally,courty2016optimal}, graphics~\cite{2015-solomon-siggraph,2016-bonneel-barycoord}, imaging~\cite{rabin2015convex,2016-Cuturi-siims}, generative models~\cite{WassersteinGAN,salimans2018improving}, biology~\cite{hashimoto2016learning,schiebinger2019optimal} or NLP~\cite{grave2018unsupervised,alaux2018unsupervised}. The key to using OT in these applications lies in the different forms of regularization of the original OT problem, as introduced in references~\cite{Villani09, SantambrogioBook}. Adding a small convex regularization to the classical linear cost not only helps on the algorithmic side, by convexifying the objective and allowing for faster solvers, but also introduces a regularity trade-off that prevents from overfitting on data measures.

\paragraph{Regularizing OT}
Although entropy-regularized OT is the most studied regularization of OT, due to its algorithmic advantages~\cite{CuturiSinkhorn}, several other convex regularizations of the transport plan have been proposed in the community: quadratically-regularized OT~\cite{essid2017quadratically}, OT with capacity constraints~\cite{korman2015optimal}, Group-Lasso regularized OT~\cite{courty2016optimal}, OT with Laplacian regularization~\cite{flamary2014optimal}, Tsallis Regularized OT~\cite{muzellec2017tsallis}, among others. On the other hand, regularizing the dual Kantorovich problem was shown in~\cite{liero2018optimal} to be equivalent to unbalanced OT, that is optimal transport with relaxed marginal constraints.

\paragraph{Understanding why regularization helps}
The question of understanding why regularizing OT proves critical has triggered several approaches. A compelling reason is statistical: Although classical OT suffers from the curse of dimensionality, as its empirical version converges at a rate of order $(1/n)^{1/d}$~\cite{dudley1969speed,fournier2015rate,weed2019sharp}, regularized OT and more precisely Sinkhorn divergences have a sample complexity of $O(1/\sqrt{n})$~\cite{pmlr-v89-genevay19a, mena2019statistical}. Entropic OT was also shown to perform maximum likelihood estimation in the Gaussian deconvolution model~\cite{rigollet2018entropic}. Taking another approach, \cite{dessein2018regularized,blondelsmoothandsparse} have considered general classes of convex regularizations and characterized them from a more geometrical perspective.


\paragraph{Robustness}
Recently, several papers~\cite{genevay2018learning, flamary2018wasserstein,deshpande2019max, kolouri2019generalized,niles2019estimation,patySRW} have proposed to maximize OT with respect to the ground cost, which can in turn be interpreted in light of ground metric learning~\cite{CuturiGroundMetric2014}. This approach can also be viewed as an instance of robust optimization~\cite{ben1998robust, ben2009robust, bertsimas2011theory}: instead of considering a data-dependent, hence unstable minimization problem $\min_x f_{\hat\theta}(x)$ where $\hat\theta$ represents the data, the robust optimization literature adversarially chooses the parameters $\theta$ in a neighborhood of the data: $\max_{\theta \in \Theta} \min_x f(x)$. Continuing along these lines, we make a connection between \emph{regularizing} and \emph{maximizing} OT.

\paragraph{Contributions}
Our main goal is to provide a novel interpretation of regularized optimal transport in terms of ground cost robustness: regularizing OT amounts to maximizing \textbf{un}regularized OT with respect to the ground cost. Our contributions are:
\begin{enumerate}
    \item We show that any convex regularization of the transport plan corresponds to ground-cost robustness (\S~\ref{sec:adversarial});
    \item We reinterpret classical regularizations of OT in the ground-cost adversarial setting (\S~\ref{sec:examples});
    \item We prove, under some technical assumption, a duality theorem for regularized OT, which we use to show that under the same assumption, there exists an optimal adversarial ground-cost that is separable (\S~\ref{sec:properties});
    \item We extend ground-cost robustness to the case of more than two measures (\S~\ref{sec:multi});
    \item We propose algorithms to solve the above-mentioned problems (\S\ref{sec:algos}) and illustrate them on data (\S~\ref{sec:experiments}).
\end{enumerate}
\section{Background on Optimal Transport and Notations}\label{sec:background}

Let $\X$ be a compact Hausdorff space, and define $\PX$ the set of Borel probability measures over $\X$. We write $\CX$ for the set of continuous functions from $\X$ to $\R$, endowed with the supremum norm. For $\phi, \psi \in \CX$, we write $\phi \oplus \psi \in \CXX$ for the function $\phi \oplus \psi : (x,y) \mapsto \phi(x) + \psi(y)$.

For $n \in \N$, we write $\range{n} = \{1, ..., n\}$. All vectors will be denoted with \textbf{bold} symbols. For a Boolean assertion $A$, we write $\iota(A)$ for its indicator function $\iota(A) = 0$ if $A$ is true and $\iota(A) = +\infty$ otherwise.

\paragraph{Kantorovich Formulation of OT}
For $\mu, \nu \in \PX$, we write $\Pi(\mu, \nu)$ for the set of couplings
\begin{multline*}
    \Pi(\mu, \nu) = \{ \pi \in \PXX \textrm{ s.t.} \,\forall A, B\subset\X \text{ Borel},\\
        \pi(A\times \X)=\mu(A), \pi(\X \times B)=\nu(B) \}.
\end{multline*}

For a real-valued continuous function $c \in \CXX$, the optimal transport cost between $\mu$ and $\nu$ is defined as
\begin{align}\label{eqn:background:kantorovich}
    \OTcost_c(\mu, \nu) := \inf_{\pi \in \Pi(\mu, \nu)} \int_{\XX} c(x,y) \,d\pi(x,y).
\end{align}
Since $c$ is continuous and $\X$ is compact, the infimum in~\eqref{eqn:background:kantorovich} is attained, see Theorem~1.4 in~\cite{SantambrogioBook}. Problem~\eqref{eqn:background:kantorovich} admits the following dual formulation, see Proposition~1.11 and Theorem~1.39 in~\cite{SantambrogioBook}:
\begin{align}\label{eqn:background:kantorovich_duality}
    \OTcost_c(\mu, \nu) = \max_{\substack{\phi, \psi \in \CX\\\phi \oplus \psi \leq c}} \int \phi \,d\mu + \int \psi \,d\nu.
\end{align}

\paragraph{Space of Measures}
Since $\X$ is compact, the dual space of $\CXX$ is the set $\MXX$ of Borel finite signed measures over $\XX$.
For $F: \MXX \to \R$, we recall that $F$ is Fréchet-differentiable at $\pi$ if there exists $\nabla F(\pi) \in \CXX$ such that for any $h \in \MXX$, as $t \to 0$
\[
    F(\pi + th) = F(\pi) + t \int \nabla F(\pi) \,dh + o(t).
\]
Similarly, $G: \CXX \to \R$ is Fréchet-differentiable at $c$ if there exists $\nabla G(c) \in \MXX$ such that for any $h \in \CXX$, as $t \to 0$
\[
    G(c + th) = G(c) + t \int h \,d\nabla G(c) + o(t).
\]

\paragraph{Legendre–Fenchel Transformation}
For any functional $F:\MXX \to \Rinf$, we can define its convex conjugate $F^*:\CXX \to \Rinf$ and biconjugate $F^{**}:\MXX \to \Rinf$ as
\begin{align*}
    F^*(c) &:= \sup_{\pi \in \MXX} \int c \,d\pi - F(\pi), \\
    F^{**}(\pi) &:= \sup_{c \in \CXX} \int c \,d\pi - F^*(c).
\end{align*}
$F^*$ is always lower semi-continuous (lsc) and convex as the supremum of continuous linear functions.

\paragraph{Specific notations}
For $F : \MXX \to \Rinf$, we write $\dom(F) = \left\{\pi \in \MXX \,|\, F(\pi) < +\infty\right\}$ for its domain and will say that $F$ is proper if $\dom(F) \neq \emptyset$.

We denote by $\FM$ the set of proper lsc convex functions $F : \MXX \to \Rinf$, and for $\mu, \nu \in \PX$, we define the set $\FM(\mu,\nu)$ of lsc convex functions that are proper on $\Pi(\mu,\nu)$:
\[
    \FM(\mu,\nu) = \left\{F \in \FM \,|\, \exists \pi \in \Pi(\mu,\nu), F(\pi) < +\infty \right\}.
\]
\section{Ground Cost Adversarial Optimal Transport}\label{sec:adversarial}

\subsection{Definition}

Instead of considering the classical \textit{linear} formulation of optimal transport~\eqref{eqn:background:kantorovich}, we consider in this paper the following more general \textit{nonlinear} convex formulation:

\begin{definition}\label{adversarial:def:Wf}
    Let $F \in \FM$. For $\mu,\nu \in \PX$, we define:
    \begin{align}\label{eqn:adversarial:Wf}
        \Wf_F(\mu,\nu) = \inf_{\pi\in\Pi(\mu,\nu)} F(\pi).
    \end{align}
\end{definition}

When $F(\pi) = \int c \,d\pi$, problem~\eqref{eqn:adversarial:Wf} corresponds to the classical optimal transport problem defined in~\eqref{eqn:background:kantorovich} and $\Wf_F = \OTcost_c$.

\begin{lemma}
    The infimum in~\eqref{eqn:adversarial:Wf} is attained. Moreover, if $F \in \FM(\mu,\nu)$, $\Wf_F(\mu,\nu) < +\infty$.
\end{lemma}
\begin{proof}
    We can apply Weierstrass's theorem since $\Pi(\mu,\nu)$ is compact and $F$ is lsc by definition. For $F \in \FM(\mu,\nu)$, there exists $\pi_0 \in \Pi(\mu,\nu)$ such that $F(\pi_0) < +\infty$, so $\Wf_F(\mu,\nu) \leq F(\pi_0) < +\infty$.
\end{proof}

The main result of this paper is the following interpretation of problem~\eqref{eqn:adversarial:Wf} as a ground-cost adversarial OT problem:
\begin{theorem}\label{adversarial:thm:Wf_is_adv}
    For $\mu,\nu \in \PX$ and $F \in \FM(\mu,\nu)$, minimizing $F$ over $\Pi(\mu,\nu)$ is equivalent to the following convex problem:
    \begin{align}\label{eqn:adversarial:Wf_is_adv}
        \Wf_F(\mu, \nu) = \sup_{c \in \CXX} \OTcost_c(\mu,\nu) - F^*(c).
    \end{align}
\end{theorem}
\begin{proof}
    Since $F$ is proper, lsc and convex, Fenchel-Moreau theorem ensures that it is equal to its convex biconjugate $F^{**}$, so:
    \begin{align*}
        \min_{\pi \in \Pi(\mu,\nu)} F(\pi) &= \min_{\pi \in \Pi(\mu,\nu)} F^{**}(\pi)\\
        &= \adjustlimits\min_{\pi \in \Pi(\mu,\nu)} \sup_{c \in \CXX} \int c\,d\pi - F^*(c).
    \end{align*}
    Define the objective $l(\pi,c) := \int c\,d\pi - F^*(c)$. Since $F^*$ is lsc as the convex conjugate of $F$, for any $\pi \in \Pi(\mu,\nu)$, $l(\pi, \cdot)$ is usc. It is also concave as the sum of concave functions. Likewise, for any $c \in \CXX$, $l(\cdot, c)$ is continuous and convex (in fact linear). Since $\Pi(\mu,\nu)$ and $\CXX$ are convex, and $\Pi(\mu,\nu)$ is compact, we can use Sion's minimax theorem to swap the min and the sup:
    \[
        \min_{\pi \in \Pi(\mu,\nu)} F(\pi) = \adjustlimits\sup_{c \in \CXX} \min_{\pi \in \Pi(\mu,\nu)} \int c\,d\pi - F^*(c).
    \]
    Finally, $c \mapsto \OTcost_c(\mu,\nu) - F^*(c)$ is concave since $F^*$ is convex and $c \mapsto \OTcost_c(\mu,\nu)$ is concave as the minimum of linear functionals.
    \vskip-0.65cm
\end{proof}
\vskip+0.3cm
\begin{remark}
    Note that the inequality
    \[
        \Wf_F(\mu, \nu) \geq \sup_{c \in \CXX} \OTcost_c(\mu,\nu) - F^*(c)
    \]
    is in fact verified for any $F:\MXX\to\Rinf$ since $F \geq F^{**}$ is always verified.
\end{remark}

The supremum in equation~\eqref{eqn:adversarial:Wf_is_adv} is not necessarily attained. Under some regularity assumption on $F$, we show that the supremum is attained and relate the optimal couplings and the optimal ground costs:
\begin{proposition}\label{adversarial:prop:c=grad(pi)}
    Let $\mu, \nu \in \PX$ and $F \in \FM(\mu,\nu)$. Suppose that $F$ is Fréchet-differentiable on $\Pi(\mu,\nu)$. Then the supremum in~\eqref{eqn:adversarial:Wf_is_adv} is attained at $\opt{c} = \nabla F(\opt{\pi})$ where $\opt{\pi}$ is any minimizer of~\eqref{eqn:adversarial:Wf}.
    Conversely, suppose $F^*$ is Fréchet-differentiable everywhere. If $\opt{c}$ is the unique maximizer in~\eqref{eqn:adversarial:Wf_is_adv}, then $\opt{\pi} = \nabla F^*(\opt{c})$ is a minimizer of~\eqref{eqn:adversarial:Wf}.
\end{proposition}
See a proof in appendix. In section~\ref{sec:properties}, we will further characterize $\opt{c}$ for a certain class of functions $F \in \FM$.

One interesting particular case of Theorem~\ref{adversarial:thm:Wf_is_adv} is when the convex cost $\pi \mapsto F(\pi)$ is a convex regularization of the classical linear optimal transport:
\begin{corollary}\label{adversarial:thm:regularization}
    Let $c_0 \in \CXX$, $\mu,\nu \in \PX$. Let $\epsilon > 0$ and $R \in \FM(\mu,\nu)$. Then:
    \begin{align}\label{eqn:adversarial:regularized_is_adv}
        &\min_{\pi \in \Pi(\mu,\nu)} \int c_0 \,d\pi + \epsilon R(\pi) \nonumber\\
        &= \sup_{c \in \CXX} \OTcost_c(\mu,\nu) - \epsilon R^*\left(\frac{c - c_0}{\epsilon}\right).
    \end{align}
\end{corollary}
\begin{proof}
    We apply theorem~\ref{adversarial:thm:regularization} with $F(\pi) = \int c_0 \,d\pi + \epsilon R(\pi)$, for which we only need to compute the convex conjugate:
    \begin{align*}
        F^*(c) &= \sup_{\pi \in \MXX} \int c-c_0 \,d\pi - \epsilon R(\pi)\\
        &= \epsilon \sup_{\pi \in \MXX} \int \frac{c - c_0}{\epsilon} \,d\pi - R(\pi)\\
        &= \epsilon R^*\left(\frac{c - c_0}{\epsilon}\right).
    \end{align*}
    \vskip-0.81cm
\end{proof}
Corollary~\ref{adversarial:thm:regularization} shows that the ground cost $c_0$ in regularized optimal transport acts as a prior on the adversarial ground cost. Indeed, in equation~\eqref{eqn:adversarial:regularized_is_adv} the penalization term $\epsilon R^*\left(\frac{c - c_0}{\epsilon}\right)$ forces any optimal adversarial ground cost to be ``close'' to $c_0$, the closeness being measured in terms of the convex conjugate of the regularization: $R^*$.

\begin{remark}\label{adv:rmk:concave}
    We can also consider the minimization of a proper usc concave function $F$ over $\Pi(\mu,\nu)$. Since $-F \in \FM$, by reusing the argument of the proof of Theorem~\ref{adversarial:thm:Wf_is_adv} (see a proof in appendix):
    \begin{align*}
        \inf_{\pi \in \Pi(\mu,\nu)} F(\pi) = \inf_{c \in \CXX} \OTcost_c(\mu,\nu) + (-F)^{*}(-c).
    \end{align*}
    Minimizing a concave function of the transport plan $\pi \in \Pi(\mu,\nu)$, or equivalently maximizing a convex function of $\pi$, amounts to finding a ground cost $c \in \CXX$ that \textit{minimizes} the transport cost between $\mu$ and $\nu$ plus a convex penalization on $c$. Note that this is not a convex problem since the objective is the sum of a concave and a convex functions. When $\mu$ and $\nu$ are discrete measures, $\Pi(\mu,\nu)$ is a finite-dimensional compact polytope so one of its extreme points has to be a minimizer of $F$.
\end{remark}

In the ground cost maximization problem, the maximization is carried out on any continuous function $c$ on $\XX$, and in particular we do not impose that $c$ takes only nonnegative values. In other words, an optimal adversarial ground cost may take negative values, which prevents us from directly interpreting optimal adversarial ground costs as suitable dissimilarity measures over $\X$. In the following subsection, we impose that $c \geq 0$ in the adversarial problem when the space $\X$ is discrete and prove an analogue of Corollary~\ref{adversarial:thm:regularization}.

\subsection{Discrete Separable Case}\label{subsec:discreteseparable}
In this subsection, we will focus on the discrete case where the space $\X = \range{n}$ for some $n \in \N$. A probability measure $\mu \in \PX$ is then a histogram of size $n$ that we will represent by a vector $\bmu \in \Rn_+$ such that $\sum_{i=1}^n \bmu_i = 1$. Cost functions $c \in \CXX$ and transport plans $\pi \in \Pi(\bmu,\bnu)$ are now matrices $\bc, \bpi \in \Rnn$.

We focus on regularization functions $R$ that are separable, \emph{i.e.} of the form
\[
    R(\bpi) = \sum_{i=1}^n \sum_{j=1}^n R_{ij}(\bpi_{ij})
\]
for some differentiable convex proper lsc $R_{ij}: \R \to \R$.

In applications, it is natural to constrain the adversarial ground cost $\bc \in \Rnn$ to take nonnegative entries. Adding this constraint on the adversarial cost corresponds to linearizing ``at short range'' the regularization $R$ for ``small transport values''.
\vspace{-0.4cm}
\begin{figure}[!h]
    \centering
    \includegraphics[width=0.49\textwidth]{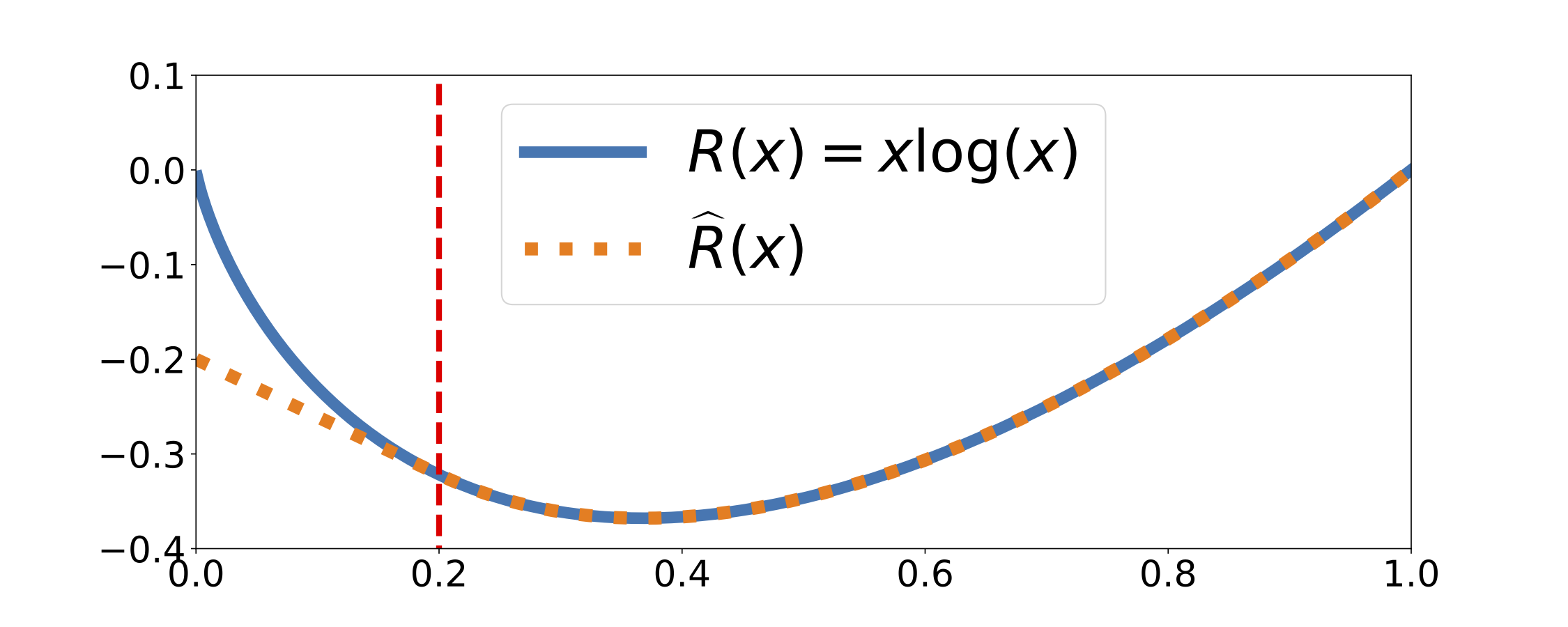}
    \vskip-0.5cm
    \caption{The entropy regularization $R(x) = x \log(x)$ and its linearized version $\widehat R(x)$ for small transport values.}
    \label{fig:adversarial:linearized_entropy}
\end{figure}

\begin{proposition}\label{adversarial:prop:discrete_positivecost}
    Let $\epsilon > 0$. For $\bmu,\bnu \in \PX$, it holds:
    \begin{align}\label{eqn:adversarial:discreteseparable:linearized}
        &\sup_{\bc \in \Rnn_+} \OTcost_{\bc}(\bmu, \bnu) - \epsilon \sum_{ij} R_{ij}^* \left(\frac{\bc_{ij} - {\bc_0}_{ij}}{\epsilon}\right) \nonumber\\
        &= \min_{\bpi \in \Pi(\bmu,\bnu)} \langle \bc_0, \bpi \rangle + \epsilon \sum_{ij} \widehat R_{ij}(\bpi_{ij})
    \end{align}
    where $\widehat R_{ij} : \R \to \R$ is the continuous convex function defined as
    \[
        \widehat R_{ij}(x) := 
        \begin{cases}
            R_{ij}(x) & \text{if } x \geq {R_{ij}^*}'\left( -\frac{{\bc_0}_{ij}}{\epsilon} \right)\\
            \frac{-{\bc_0}_{ij}}{\epsilon} x - R_{ij}^*\left( -\frac{{\bc_0}_{ij}}{\epsilon} \right) & \text{otherwise.}
        \end{cases}
    \]
    Moreover, if $R_{ij}$ is of class $C^1$, then $\widehat R_{ij}$ is also $C^1$.
\end{proposition}
\section{Examples}\label{sec:examples}

\subsection{Ground Cost Adversarial Interpretation of Classical OT Regularizations}\label{subsec:examples:examples}

As presented in the introduction, several convex regularizations $R$ have been proposed in the literature. We give the ground cost adversarial counterpart for some of them: two examples in the continuous setting, and four $p$-norm based regularizations in the discrete case.

\begin{example}[Entropic Regularization]\label{adversarial:example:entropic}
    Let $\mu, \nu \in \PX$. For $\pi \in \Pi(\mu,\nu)$, we define its relative entropy as $\KL(\pi \| \mu\otimes\nu) = \int \log\frac{d\pi}{d\mu\otimes\nu} d\pi$.
    Then for $c_0 \in \CXX$ and $\epsilon > 0$, it holds:
    \begin{align*}
        &\min_{\pi \in \Pi(\mu,\nu)} \int c_0 \,d\pi + \epsilon \KL(\pi \| \mu\otimes\nu) \\
        &= \sup_{c \in \CXX} \OTcost_c(\mu, \nu) - \epsilon \int \exp\left(\frac{c - c_0}{\epsilon}\right) \,d\mu\otimes\nu + \epsilon.
    \end{align*}
\end{example}
\begin{proof}
    For $\pi \in \MXX$, let 
    \[
        R(\pi) = \begin{cases}
                    \int \log \frac{d\pi}{d\mu\otimes\nu} d\pi - \int d\pi + 1 & \text{if } \pi \ll \mu\otimes\nu\\
                    +\infty & \text{otherwise.}
                \end{cases}
    \]
    $R$ is convex, and using proposition~7 in~\cite{feydyinterpolating},
    \[
        R^*(c) = \int e^c - 1 \, d\mu\otimes\nu.
    \]
    Applying corollary~\ref{adversarial:thm:regularization} concludes the proof.
\end{proof}

Another case of interest is the so-called Subspace Robust Wasserstein distance recently proposed by~\cite{patySRW}. Here, the set of adversarial metrics is parameterized by a finite-dimensional parameter $\Omega$, which allows to recover an adversarial metric defined on the whole space even when the measures are finitely supported.
\begin{example}[Subspace Robust Wasserstein]\label{adversarial:example:SRW}
    Let $d \in \N$, $k \in \range{d}$ and $\mu, \nu \in \PRd$ with a finite second-order moment. For $\pi \in \Pi(\mu,\nu)$, define $V_\pi = \int (x-y)(x-y)^\top d\pi(x,y)$ and $\lambda_1(V_\pi) \geq \ldots \geq \lambda_d(V_\pi)$ its ordered eigenvalues.
    
    Then $F: \pi \mapsto \sum_{l=1}^k \lambda_l(V_\pi)$ is convex, and
    \[
        \SRW_k(\mu, \nu) := \min_{\pi \in \Pi(\mu,\nu)} \sum_{l=1}^k \lambda_l(V_\pi) = \max_{\substack{0 \preceq \Omega \preceq I\\\trace(\Omega)=k}} \OTcost_{d_\Omega^2}(\mu,\nu)
    \]
    where $d_\Omega^2(x,y) = (x-y)^\top \Omega (x-y)$ is the squared Mahalanobis distance.
\end{example}
\begin{proof}
    See Theorem~1 in~\cite{patySRW}. Note that in this case, $\X = \Rd$ is not compact. This is not a problem since $F^* \equiv +\infty$ outside a compact set, \emph{i.e.} the set on metrics on which the maximization takes place is compact. Indeed, one can show that:
    \[
        \!\!\!\!F^*(c) = \iota(\exists 0 \preceq \Omega \preceq I \text{ with } \trace(\Omega)=k \text{ s.t. } c=d_\Omega^2).
    \]
    \vskip-0.90cm
\end{proof}

Let us now consider $p$-norm based examples, which will subsume quadratically-regularized ($p=2$) OT studied in~\cite{essid2017quadratically,lorenz2019quadratically}, capacity-constrained ($p=+\infty$) OT proposed by~\cite{korman2015optimal} and Tsallis regularized ($p < 0$) OT introduced by~\cite{muzellec2017tsallis}.

For a matrix $\bw \in \Rnn_+$ with $\sum_{ij} \bw_{ij} = n^2$ and $\bpi \in \Rnn$, we denote by $\|\bpi\|_{\bw, p}^p = \sum_{ij} \bw_{ij} | \bpi_{ij} |^p$ the $\bw$-weighted (powered) $p$-norm of $\bpi$. We also write $1/\bw$ for the matrix defined by $(1/\bw)_{ij} = 1/\bw_{ij}$. In the following, except otherwise mentioned, we take $p, q \in [1, +\infty]$ such that $1/p + 1/q = 1$, $\bc_0 \in \Rnn$, $\epsilon > 0$.
\begin{example}[$\|\cdot\|_{\bw,p}^p$ Regularization]\label{adversarial:ex:pp-reg}
    \begin{align*}
        \min_{\bpi \in \Pi(\bmu, \bnu)} &\langle \bc_0, \bpi \rangle + \epsilon \frac{1}{p} \|\bpi\|_{\bw,p}^p \\
        &= \sup_{\bc \in \Rnn} \OTcost_{\bc}(\bmu, \bnu) - \epsilon\frac{1}{q} \left\|\frac{\bc - \bc_0}{\epsilon}\right\|_{{1/\bw}^{q-1},q}^q.
    \end{align*}
    In particular when $p=2$ and $\bw = 1$, this corresponds to quadratically-regularized OT studied in~\cite{essid2017quadratically,lorenz2019quadratically}.
\end{example}
We give the details of the (straightforward) computations in the appendix.

\begin{example}[$\|\cdot\|_{\bw,p}$ Penalization]
    \[
        \min_{\bpi \in \Pi(\bmu, \bnu)} \langle \bc_0, \bpi \rangle + \epsilon \|\bpi\|_{\bw, p} 
        = \sup_{\substack{\bc \in \Rnn\\ \|\bc - \bc_0\|_{{1/\bw}, q} \leq \epsilon}} \OTcost_{\bc}(\bmu, \bnu).
    \]
\end{example}
\begin{proof}
    We apply Corollary~\ref{adversarial:thm:regularization} with $R: \Rnn \to \Rnn$ defined as $R(\bpi) = \|\bpi\|_{\bw, p}$, for which we need to compute its convex conjugate. We know that the dual of $\|\cdot\|_p$ is $\iota(\|\cdot\|_q \leq 1)$, and using classical results about convex conjugates, $\|\cdot\|_{\bw, p}^* = \iota(\|\cdot\|_{1/\bw, q} \leq 1)$.
\end{proof}

\begin{example}[$\|\cdot\|_{\bw,p}$ Regularization]
    \[
        \min_{\substack{\bpi \in \Pi(\bmu,\bnu)\\\|\bpi\|_{\bw,p} \leq \epsilon}} \langle \bc_0, \bpi \rangle 
        = \sup_{\bc \in \Rnn} \OTcost_{\bc}(\bmu, \bnu) - \epsilon \|\bc - \bc_0\|_{1/\bw,q}.
    \]
    In particular when $p=+\infty$ and $\bw = 1$, this coincides with capacity-constrained OT proposed by~\cite{korman2015optimal}.
\end{example}
\begin{proof}
    We apply Corollary~\ref{adversarial:thm:regularization} with $R: \Rnn \to \Rnn$ defined as $R(\bpi) = \iota(\|\bpi\|_{\bw, p} \leq 1)$, for which we need to compute its convex conjugate. We know that the dual of $\iota(\|\cdot\|_p \leq 1)$ is $\|\cdot\|_q$, and using classical results about convex conjugates, $\iota(\|\cdot\|_{\bw, p} \leq 1)^* = \|\cdot\|_{1/\bw, q}$.
\end{proof}

\begin{example}[Tsallis Regularization]\label{adversarial:ex:tsallis}
    For $q \in (0, 1)$, the Tsallis regularized OT problem~\cite{muzellec2017tsallis}
    \[
        \min_{\substack{\bpi \in \Pi(\bmu,\bnu)}} \langle \bc_0, \bpi \rangle - \epsilon \frac{1}{1-q} \sum_{ij} \left( \bpi_{ij}^q - \bpi_{ij} \right)
    \]
    is equivalent to
    \[
        \sup_{\substack{\bc \in \Rnn\\\bc \leq \bc_0}}
        \OTcost_{\bc}(\bmu, \bnu) 
        - \epsilon^{\frac{1}{1-q}} (-p)^{-p}
        \left\| \frac{1}{\bc_0 - \bc} \right\|_{-p}^{-p} 
        + \frac{\epsilon}{1-q}
    \]
    where $p < 0$ is such that $1/p + 1/q = 1$.
\end{example}
We give the details of the computations in appendix.

\subsection{A Link With the Matching Literature in Economics}

Maximizing the OT problem with respect to the ground cost has been proposed in the matching literature in economics as a way to recover a ground cost when only a matching is observed, see \textit{e.g.}~\cite{dupuy2014personality, galichon2015cupid,dupuy2016estimating}. In this subsection, we reinterpret their methods by showing that they are equivalent to some regularized OT problems. In other words, instead of interpreting a regularization problem as a robust OT problem as in subsection~\ref{subsec:examples:examples}, we go the other way around and show that this practical OT maximization problem corresponds to a regularized OT problem.

Practitioners observe two probability measures $\mu,\nu \in \PX$ (\textit{e.g.} features from a group of men and a group of women) and a matching $\pi_0 \in \Pi(\mu,\nu)$ (\textit{e.g.} dating or marriage data). Under the assumption that the matching is optimal for some criteria, we can determine these by finding a ground cost $\opt{c} \in \CXX$ such that the matching $\pi_0$ is an optimal transport plan for the cost $c$. Then $\opt{c}(x,y)$ can be interpreted as the unwillingness for two people with characteristics $x$ and $y$ to be matched.

As shown in Theorem~3 in~\cite{galichon2015cupid},
\begin{align}\label{eqn:maxcost_economists}
    \sup_{c \in \CXX} \OTcost_c(\mu,\nu) - \int c \, d{\pi_0}
    =
    \iota\left(\pi_0 \in \Pi(\mu,\nu)\right)
\end{align}
and if $\pi_0 \in \Pi(\mu,\nu)$, the supremum is attained at any $\opt{c} \in \CXX$ such that $\pi_0$ is an optimal transport plan for the cost $\opt{c}$. Indeed, the first order condition for the maximization problem and the envelope theorem give the result.

In practice, economists are more interested in discovering which features explain the most the observed matching $\pi_0$. To this end, they choose a parametric model for the cost $c$, for example a Mahalanobis model $c \in \left\{ d_\Omega^2: (x,y) \mapsto (x-y)^\top \Omega (x-y) \,|\, \Omega \succeq 0, \|\Omega\| \leq 1 \right\}$. More generally, we can rewrite problem~\eqref{eqn:maxcost_economists} as
\begin{align}\label{eqn:maxcost_economists_regularized}
    \sup_{c \in \CXX} \OTcost_c(\mu,\nu) - \int c \, d{\pi_0} - R^*(c)
\end{align}
where $R \in \FM$ is a lsc convex functional, \textit{e.g.} $R^*(c) = \iota\left(\exists \Omega \succeq 0, \|\Omega\|\leq 1, c = d_\Omega^2\right)$ for the Mahalanobis model.

Using Theorem~\ref{adversarial:thm:Wf_is_adv}, we can then reinterpret problem~\eqref{eqn:maxcost_economists_regularized}:
\begin{align*}
    &\sup_{c \in \CXX} \OTcost_c(\mu,\nu) - \int c \, d{\pi_0} - R^*(c) \nonumber\\
    &=
    \min_{\pi \in \Pi(\mu,\nu)} R(\pi - \pi_0)
\end{align*}
where we have used the fact that $R^{**} = R$. Solving equation~\eqref{eqn:maxcost_economists_regularized} amounts to finding a matching $\pi \in \Pi(\mu,\nu)$ that is close to the observed matching $\pi_0$, as measured by $R$.
\section{Characterization of the Adversarial Cost and Duality}\label{sec:properties}

Theorem~\ref{adversarial:thm:Wf_is_adv} shows that regularizing OT is equivalent to maximizing unregularized OT with respect to the ground cost. This gives access to a robustly computed ground-cost $\opt{c}$. In this section, we first prove a duality theorem for problem~\eqref{eqn:adversarial:Wf} that we use to further characterize $\opt{c}$. We will first need a technical assumption on $F$:
\begin{definition}
    Let $F \in \FM$. We will say that $F$ is \emph{separably $*$-increasing} if for any $\phi, \psi \in \CX$ and any $c \in \CXX$:
    \begin{align}\label{eqn:properties:star_increasing}
        \phi \oplus \psi \leq c \Rightarrow F^*(\phi \oplus \psi) \leq F^*(c).
    \end{align}
    In particular if $F^*$ is increasing, $F$ is separably $*$-increasing.
\end{definition}

This definition, albeit not always verified \textit{e.g.} in the discrete separable case of Proposition~\ref{adversarial:prop:discrete_positivecost} and in the SRW case of Example~\ref{adversarial:example:SRW}, is indeed verified in various cases of interest, \textit{e.g.} for the entropic or $\|\cdot\|_{\bw,p}^p$ regularizations:
\begin{example}
    For $\mu,\nu \in \PX$, $c_0 \in \CXX$ and $\epsilon > 0$, the entropy-regularized OT function
    \[
        F: \pi \mapsto \int c_0 \,d\pi + \epsilon \KL(\pi \| \mu\otimes\nu)
    \]
    is separably $*$-increasing.
\end{example}
\begin{proof}
    As in the proof of example~\ref{adversarial:example:entropic},
    \[
        F^*(c) = \epsilon \int \exp\left(\frac{c-c_0}{\epsilon}\right) - 1 \,d\mu\otimes\nu
    \]
    which verifies condition~\eqref{eqn:properties:star_increasing} as an increasing functional.
\end{proof}
\begin{example}\label{properties:example:norm_p}
    In the discrete setting $\X = \range{n}$, let $\bmu,\bnu \in \PX$, $\bc_0 \in \Rnn$, $\bw \in \Rnn_+$ summing to $n^2$. Take $p>1$ and $\epsilon > 0$. With $\varphi_p(x) = x^p$ if $x \geq 0$ and $\varphi_p(x) = +\infty$ if $x<0$, the $\|\cdot\|_{\bw,p}^p$-regularized OT function
    \[
        F: \bpi \mapsto \langle \bc_0, \bpi \rangle + \epsilon \sum_{ij} \bw_{ij} \varphi_p(\bpi_{ij})
    \]
    is separably $*$-increasing.
\end{example}
\begin{proof}
    Note that minimizing $F$ over $\Pi(\bmu,\bnu) \subset \Rnn_+$ is equivalent to minimizing $\widetilde F: \bpi \mapsto \langle \bc_0, \bpi \rangle + \epsilon \sum_{ij} \bw_{ij} |\bpi_{ij}|^p$.
    One can show that, with $q > 1$ such that $1/p + 1/q = 1$ and $(x)_+ := \max\{0, x\}$:
    \[
        F^*(\bc) = \epsilon \frac{1}{q} \left\| \frac{(\bc - \bc_0)_+}{\epsilon} \right\|_{1/\bw^{q-1}, q}^q
    \]
    which clearly verifies condition~\eqref{eqn:properties:star_increasing}.
\end{proof}

When $F$ is separably $*$-increasing, we can easily prove a duality theorem for problem~\eqref{eqn:adversarial:Wf}:
\begin{theorem}[$\Wf_F$ duality]\label{prop:thm:duality}
    Let $\mu,\nu \in \PX$ and $F \in \FM(\mu,\nu)$ a separably $*$-increasing function. Then:
    \begin{align}\label{eqn:properties:duality}
        \!\!\!\!\!\Wf_F(\mu,\nu) = \max_{\phi, \psi \in \CX} \int \phi d\mu + \int \psi d\nu - F^*(\phi \oplus \psi).
    \end{align}
\end{theorem}
\begin{proof}
    Using Theorem~\ref{adversarial:thm:Wf_is_adv} and Kantorovich duality~\eqref{eqn:background:kantorovich_duality}:
    \begin{align*}
        \!\!\Wf_F(\mu,\nu) &= \sup_{c \in \CXX} \OTcost_c(\mu,\nu) - F^*(c)\\
        &= \sup_{c \in \CXX} \max_{\substack{\phi,\psi \in \CX\\\phi\oplus\psi \leq c}} \int \phi \,d\mu + \int \psi \,d\nu - F^*(c)\\
        &= \sup_{c \in \CXX} \max_{\phi,\psi \in \CX} \int \phi \,d\mu + \int \psi \,d\nu - F^*(c) \\
        & \phantom{aaaaaaaaaaaaaaaaa} - \iota(\phi\oplus\psi \leq c)\\
        &= \max_{\phi,\psi \in \CX} \int \phi \,d\mu + \int \psi \,d\nu\\
        &\phantom{aaaaaaaaa} + \sup_{c \in \CXX} - F^*(c) - \iota(\phi\oplus\psi \leq c)\\
        &= \max_{\phi, \psi \in \CX} \int \phi \,d\mu + \int \psi \,d\nu - \inf_{\substack{c\in \CXX\\\phi \oplus \psi \leq c}} F^*(c).
    \end{align*}
    Since $F$ is separably $*$-increasing, for any $\phi, \psi \in \CX$,
    \[
        \inf_{\substack{c\in \CXX\\\phi \oplus \psi \leq c}} F^*(c) = F^*(\phi \oplus \psi),
    \]
    which shows the desired duality result.
\end{proof}

Theorem~\ref{prop:thm:duality} subsumes the already known duality results for entropy-regularized OT and quadratically-regularized OT. It also enables us to characterize of the optimal adversarial ground cost when the convex objective $F \in \FM$ is separably $*$-increasing:
\begin{corollary}\label{properties:corollary:c=phi+psi}
    If $\opt{\phi}, \opt{\psi}$ are optimal solutions in~\eqref{eqn:properties:duality}, the cost $\opt{\phi} \oplus \opt{\psi} \in \CXX$ is an optimal adversarial cost in~\eqref{eqn:adversarial:Wf_is_adv}.
\end{corollary}
\begin{proof}
    For $\phi, \psi \in \CX$, note that
    \[
        \OTcost_{\phi\oplus\psi}(\mu, \nu) = \int \phi \,d\mu + \int \psi \,d\nu.
    \]
    Then using $\Wf_F$ duality:
    \begin{align*}
        \Wf_F(\mu, \nu) &= \max_{\phi, \psi \in \CX} \int \phi \,d\mu + \int \psi \,d\nu - F^*(\phi \oplus \psi) \\
        &= \max_{\phi, \psi \in \CX} \OTcost_{\phi\oplus\psi}(\mu, \nu) - F^*(\phi \oplus \psi) \\
        &\leq \sup_{c \in \CXX} \OTcost_c(\mu, \nu) - F^*(c) \\
        &= \Wf_F(\mu, \nu)
    \end{align*}
    where we have used Theorem~\ref{adversarial:thm:Wf_is_adv} in the last line. This shows that the inequality is in fact an equality, so if $\opt{\phi}, \opt{\psi}$ are optimal dual potentials in~\eqref{eqn:properties:duality}, $\opt{\phi} \oplus \opt{\psi}$ is an optimal adversarial cost in~\eqref{eqn:adversarial:Wf_is_adv}.
\end{proof}

Corollary~\ref{properties:corollary:c=phi+psi} is quite striking. Indeed, in the regularized formulation of Corollary~\ref{adversarial:thm:regularization}, any optimal ground cost $\opt{c}$ in equation~\eqref{eqn:adversarial:regularized_is_adv} should be close (in $R^*$ sense) to the prior cost $c_0$ because of the penalization term $\epsilon R^*\left(\frac{c-c_0}{\epsilon}\right)$. But under the assumption that $F$ is separably $*$-increasing, we have just shown that regardless of $c_0$, there exists an optimal adversarial ground cost that is separable.
\section{Adversarial Ground-Cost for Several Measures}\label{sec:multi}

For two measures $\mu,\nu \in \PX$ and a separably $*$-increasing function $F \in \FM(\mu,\nu)$, corollary~\ref{properties:corollary:c=phi+psi} shows that there exists an optimal adversarial ground cost $\opt{c}$ that is separable. This separability, which is verified \textit{e.g.} in the entropic or quadratic case, means that the OT problem for $\opt{c}$ is degenerate in the sense that any transport plan is optimal for the cost $\opt{c}$. From a metric learning point of view, $\opt{c}$ is not a suitable dissimilarity measure on $\X$. But why limit ourselves to two measures? If we observe $N \in \N$ measures $\mu_1, \ldots, \mu_N \in \PX$, we could look for a ground cost $c \in \CXX$ that is adversarial to all the pairs:
\[
    \sup_{c \in \CXX} \sum_{i \neq j} \OTcost_c(\mu_i, \mu_j) - F^*(c)
\]
for some convex regularization $F^*: \CXX \to \Rinf$. We will specifically focus on the case where we observe a sequence of measures $\mu_{1:T} := \mu_1, \ldots, \mu_T \in \PX$, $T \geq 2$. When we observe such time-dependent data, we can look for a sequence of adversarial costs $c_{1:T-1} := c_1, \ldots, c_{T-1} \in \CXX$ which is globally adversarial:
\begin{definition}\label{multi:def:WR}
    For $D: \CXX \times \CXX \to \Rinf$ and $F_t \in \FM(\mu_t, \mu_{t+1})$, $t \in \range{T-1}$, we define:
    \begin{align}\label{eqn:multi:def:multi}
        \WR_{D, F}(\mu_{1:T}) := &\sup_{c_{1:T-1}} \sum_{t=1}^{T-1} \OTcost_{c_t}(\mu_t, \mu_{t+1})\\
        &\qquad \qquad - D(c_t, c_{t+1}) - F_t^*(c_t)\nonumber
    \end{align}
    with the convention $D(c_{T-1}, c_T) = 0.$
\end{definition}
In problem~\eqref{eqn:multi:def:multi}, $D$ acts as a time-regularization by forcing the adversarial sequence of ground-costs to vary ``continuously'' with time.

Taking inspiration from the Subspace Robust Wasserstein (SRW) distance, we propose as a particular case of definition~\ref{multi:def:WR} a generalization of SRW to the case of a sequence of measures $\mu_1, \ldots, \mu_T$, $T \geq 2$:
\begin{definition}\label{multi:def:timeSRW}
    Let $d \in \N$ and $k \in \range{d}$.
    Define $\mathcal{R}_k = \left\{ \Omega \in \Rdd \,|\, 0 \preceq \Omega \preceq I, \trace(\Omega)=k \right\}$.
    We define the \emph{sequential SRW} between $\mu_1, \ldots, \mu_T \in \PRd$ as:
    \begin{align}\label{eqn:multi:def:timeSRW}
        \tSRW_{k,\eta}(\mu_{1:T}) := &\sup_{\Omega_1, \ldots, \Omega_{T-1} \in \mathcal{R}_k} \sum_{t=1}^{T-1} \OTcost_{d_{\Omega_t}^2}(\mu_t, \mu_{t+1}) \\
        &\qquad\qquad\qquad\qquad - \eta \Bures(\Omega_t, \Omega_{t+1})\nonumber
    \end{align}
    where $\Bures(A, B) = \trace(A + B - 2(A^{\frac{1}{2}} B A^{\frac{1}{2}})^{\frac{1}{2}})$ is the squared Bures metric~\citep{bures1969extension,bhatia2018bures} on the SDP cone.
\end{definition}

Note that problem~\eqref{eqn:multi:def:timeSRW} is convex. If $T=2$, the sequential SRW is equal to the usual SRW distance: $\tSRW_{k,\eta}(\mu_1, \mu_2) = \SRW_k(\mu_1, \mu_2)$.
\section{Algorithms}\label{sec:algos}

From now on, we only consider the discrete case $\X = \range{n}$.

\subsection{Projected (Sub)gradient Ascent Solves Nonnegative Adversarial Cost OT}

In the setting of subsection~\ref{subsec:discreteseparable}, we propose to run a projected subgradient ascent on the ground cost $\bc \in \Rnn_+$ to solve problem~\eqref{eqn:adversarial:discreteseparable:linearized}. Note that in this case, $\widehat F(\bpi) := \langle \bc_0, \bpi \rangle + \epsilon \sum_{ij} \widehat R_{ij}^*\left(\frac{\bc_{ij} - {\bc_0}_{ij}}{\epsilon}\right)$ is \textbf{not} separably $*$-increasing, so we can hope that the optimal adversarial ground cost will not be separable.

At each iteration of the ascent, we need to compute a subgradient of $g: \bc \mapsto \OTcost_\bc(\bmu,\bnu) - \epsilon R^*\left(\frac{\bc - {\bc_0}}{\epsilon}\right)$ given by Danskin's theorem:
\begin{align*}
    &\partial g(\bc) =\\
    &\conv \left\{\opt{\bpi} - \nabla R^*\left(\frac{\bc - \bc_0}{\epsilon}\right)
    \,\Bigg|\, \opt{\bpi} \in \argmin_{\bpi \in \Pi(\bmu,\bnu)} \langle \bc, \bpi \rangle \right\}.
\end{align*}

Although projected subgradient ascent does converge, having access to gradients instead of subgradients, hence regularity, helps the convergence. We therefore propose to replace $\OTcost_\bc(\bmu,\bnu)$ by its entropy-regularized version
\[
    \Sinkhorn^\eta_{\bc}(\mu,\nu) = \min_{\bpi \in \Pi(\bmu,\bnu)} \langle \bc, \bpi \rangle + \eta \sum_{ij} \bpi_{ij} ( \log \bpi_{ij} - 1)
\]
in the definition of the obective $g$. Then $g$ is differentiable, because there exists a unique solution $\opt{\bpi}$ in the entropic case (hence $\partial g(\bc)$ is a singleton). This will also speed up the computations of the gradient at each iteration using Sinkhorn's algorithm. We can interpret this addition of a small entropy term in the adversarial cost formulation as a further regularization of the primal:
\begin{corollary}\label{algo:corollary:max_C_entropy}
    Using the same notations as in Theorem~\ref{adversarial:thm:Wf_is_adv}, for $\eta \geq 0$:
    \begin{align*}
        \sup_{\bc \in \Rnn}& \Sinkhorn^\eta_{\bc}(\bmu,\bnu) - F^*(\bc)\\
        &= \min_{\bpi \in \Pi(\bmu,\bnu)} F(\bpi) + \eta \sum_{ij} \bpi_{ij} ( \log \bpi_{ij} - 1).
    \end{align*}
\end{corollary}
\begin{proof}
    Let $R(\bpi) := \sum_{ij} \bpi_{ij} ( \log \bpi_{ij} - 1)$. Then:
    \begin{align*}
        \sup_{\bc \in \Rnn} &\Sinkhorn^\eta_{\bc}(\bmu,\bnu) - F^*(\bc) \\
        &= \sup_{\bc \in \Rnn} \min_{\bpi \in \Pi(\bmu, \bnu)} \langle \bpi, \bc \rangle + \eta R(\bpi) - F^*(\bc) \\
        &= \min_{\bpi \in \Pi(\bmu, \bnu)} \eta R(\bpi) + \sup_{\bc \in \Rnn} \langle \bpi, \bc \rangle - F^*(\bc) \\
        &= \min_{\bpi \in \Pi(\bmu, \bnu)} \eta R(\bpi) + F(\bpi)
    \end{align*}
    where we have used Sion's minimax theorem as in the proof of Theorem~\ref{adversarial:thm:Wf_is_adv} to swap the min and the sup, and used as well the fact that $F = F^{**}$ given by Fenchel-Moreau theorem.
\end{proof}

\begin{algorithm}[!t]
	\caption{Projected \textit{(sub)}Gradient Ascent for Nonnegative Adversarial Cost}
	\label{alg:gradienteascent}
	\begin{algorithmic}
   		\STATE {\bfseries Input:} Histograms $\bmu, \bnu \in \Rn$, learning rate $\mathrm{lr}$
		\STATE Initialize $\bc \in \Rnn_+$
   		\FOR{$i=0$ {\bfseries to} MAXITER}
		\STATE $\opt{\bpi} \leftarrow$ $\mathrm{OT}(\bmu, \bnu, \mathrm{cost}=\bc)$
		\STATE $\bc \leftarrow$ $\mathrm{Proj}_{\Rnn_+}\left[ \bc + \mathrm{lr } \opt{\bpi} - \mathrm{lr } \nabla R^*\left(\frac{\bc - \bc_0}{\epsilon}\right) \right]$
  		\ENDFOR
	\end{algorithmic}
\end{algorithm}

\subsection{Sinkhorn-like Algorithm for $*$-increasing $F \in \FM$}\label{subsec:algo:sinkhornlike}
If the function $F\in\FM$ is separably $*$-increasing, we can directly write the optimality conditions for the concave dual problem~\eqref{eqn:properties:duality}:
\begin{align}
    \bmu &= \nabla F^*(\opt{\bphi} \oplus \opt{\bpsi}) \ones \label{eqn:algo:solvephi}\\
    \bnu &= \nabla F^*(\opt{\bphi} \oplus \opt{\bpsi})^\top \ones \label{eqn:algo:solvepsi}
\end{align}
where $\ones$ is the vector of all ones.
We can then alternate between fixing $\bpsi$ and solving for $\bphi$ in~\eqref{eqn:algo:solvephi} and fixing $\bphi$ and solving for $\bpsi$ in~\eqref{eqn:algo:solvepsi}. In the case of entropy-regularized OT, this is equivalent to Sinkhorn's algorithm. In quadratically-regularized OT, this is equivalent to the alternate minimization proposed by~\cite{blondelsmoothandsparse}. We give the detailed derivation of these facts in the appendix.

\subsection{Coordinate Ascent for Sequential SRW}

Problem~\eqref{eqn:multi:def:timeSRW} is a globally convex problem of $\Omega_1, \ldots, \Omega_{T-1}$. We propose to run a randomized coordinate ascent on the concave objective, \textit{i.e.} to select $\tau \in \range{T-1}$ randomly at each iteration and doing a gradient step for $\Omega_\tau$. We need to compute a subgradient of the objective $h: \Omega_\tau \mapsto \sum_{t=1}^{T-1} \OTcost_{d_{\Omega_t}^2}(\mu_t, \mu_{t+1}) - \eta \Bures(\Omega_t, \Omega_{t+1})$, given by:
\begin{align}\label{eqn:algo:grad_Omega_tau}
    \nabla h(\Omega_\tau) = V(\opt{{\bpi_\tau}}) &- \eta \partial_1 \Bures(\Omega_\tau, \Omega_{\tau+1}) \\
    & - \eta \partial_2 \Bures(\Omega_{\tau-1}, \Omega_\tau)\nonumber
\end{align}
where $\bpi \mapsto V(\bpi)$ is defined in Example~\ref{adversarial:example:SRW}, $\opt{{\bpi_\tau}} \in \Rnn$ is any optimal transport plan between $\mu_\tau, \mu_{\tau+1}$ for cost $d^2_{\Omega_\tau}$, and $\partial_1 \Bures, \partial_2 \Bures$ are the gradients of the squared Bures metric with respect to the first and second arguments, computed \textit{e.g.} in~\cite{muzellecembedding}.

\begin{algorithm}[!h]
	\caption{Randomized (Block) Coordinate Ascent for sequential SRW}
	\label{alg:coordinateascent}
	\begin{algorithmic}
  		\STATE {\bfseries Input:} Measures $\mu_1, \ldots, \mu_T \in \PRd$, dimension $k$, learning rate $\mathrm{lr}$
		\STATE Initialize $\Omega_1, \ldots, \Omega_{T-1} \in \Rdd$
  		\FOR{$i=0$ {\bfseries to} MAXITER}
  		\STATE Draw $\tau \in \range{T-1}$
		\STATE $\opt{{\pi_\tau}} \leftarrow$ $\mathrm{OT}(\mu_\tau, \mu_{\tau+1}, \mathrm{cost}=d^2_{\Omega_\tau})$
		\STATE $\Omega_\tau \leftarrow$ $\mathrm{Proj}_{\mathcal{R}_k}\left[ \Omega_\tau + \mathrm{lr} \nabla h(\Omega_\tau) \right]$ using~\eqref{eqn:algo:grad_Omega_tau}
  		\ENDFOR
	\end{algorithmic}
\end{algorithm}
\section{Experiments}\label{sec:experiments}

\subsection{Linearized Entropy-Regularized OT}
We consider the entropy-regularized OT problem in the discrete setting:
\begin{align*}
    \Sinkhorn^{\epsilon}_{\bc_0} (\bmu, \bnu) = \min_{\bpi \in \Pi(\bmu,\bnu)} \langle \bc_0, \bpi \rangle + \epsilon R(\bpi)
\end{align*}
where $\bc_0 \in \Rnn$ and $R: \bpi \mapsto \sum_{ij} \bpi_{ij} ( \log \bpi_{ij} - 1)$. Since $R$ is separable, we can constrain the associated adversarial cost to be nonnegative by linearizing the entropic regularization. By proposition~\ref{adversarial:prop:discrete_positivecost}, this amounts to solve
\begin{align}\label{eqn:expes:linearized_entropic}
    &\sup_{\bc \in \Rnn_+} \OTcost_{\bc}(\bmu, \bnu) - \epsilon \sum_{ij} \exp \left(\frac{\bc_{ij} - {\bc_0}_{ij}}{\epsilon}\right)\\
    &= \min_{\bpi \in \Pi(\bmu,\bnu)} \langle \bc_0, \bpi \rangle + \epsilon \sum_{ij} \widehat R_{ij}(\bpi_{ij}) \nonumber
\end{align}
where $\widehat R_{ij} : \R \to \R$ is defined as
\[
    \widehat R_{ij}(x) := 
    \begin{cases}
        x (\log x - 1) 
            & \text{if } x \geq \exp\left( -\frac{{\bc_0}_{ij}}{\epsilon} \right)\\
        \frac{-{\bc_0}_{ij}}{\epsilon} x - \exp\left( -\frac{{\bc_0}_{ij}}{\epsilon} \right)
            & \text{otherwise.}
    \end{cases}
\]

We first consider $N=100$ couples of measures $(\bmu_i, \bnu_i)$ in dimension $d=1000$, each measure being a uniform measure on $n=100$ samples from a Gaussian distribution with covariance matrix drawn from a Wishart distribution with $k=d$ degrees of freedom. For each couple, we run Algorithm~\ref{alg:gradienteascent} to solve problem~\eqref{eqn:expes:linearized_entropic}. This gives an adversarial cost $\opt{\bc^\epsilon}$. We plot in Figure~\ref{fig:expes:dependence_on_eps} the mean value of $\left| \widehat W_\epsilon - \OTcost_{\|\cdot\|^2}(\bmu_i,\bnu_i) \right|$ depending on $\epsilon$, for $\widehat W_\epsilon$ equal to $\OTcost_{\opt{\bc^\epsilon}}(\bmu_i, \bnu_i)$, $\Sinkhorn
^\epsilon(\bmu_i, \bnu_i)$ and the value of~\eqref{eqn:expes:linearized_entropic}. For small values of $\epsilon$, all three values converge to the real Wasserstein distance. For large $\epsilon$, Sinkhorn stabilizes to the MMD~\cite{genevay2016stochastic} while the robust cost goes to $0$ (for the adversarial cost goes to $0$).
\begin{figure}[!h]
    \centering
    \!\!\!\includegraphics[width=0.49\textwidth]{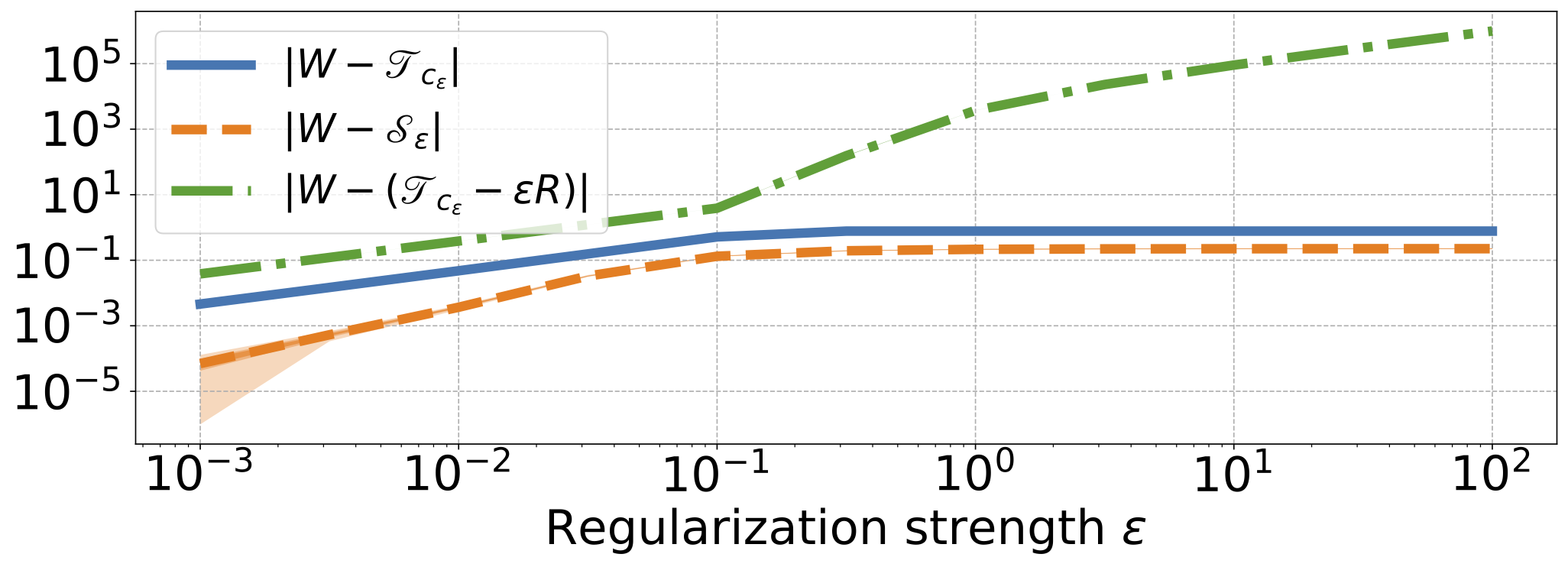}
    \vskip -0.4cm
    \caption{Mean value (over $100$ runs) of the difference between the classical (2-Wasserstein) OT cost $W$ and Sinkhorn $\Sinkhorn^\epsilon$ (orange dashed), OT cost with adversarial nonnegative cost $\OTcost_{c_\epsilon}$ (blue line) and the value of problem~\eqref{eqn:expes:linearized_entropic} $\OTcost_{c_\epsilon} - \epsilon R$ (green dot-dashed) depending on $\epsilon$. The shaded areas represent the min-max, 10\%-90\% and 25\%-75\% percentiles, and appear negligeable except for numerical errors.}
    \label{fig:expes:dependence_on_eps}
\end{figure}

In Figure~\ref{fig:expes:visualize_epsilon_cost}, we visualize the effect of the regularization $\epsilon$ on the ground cost $\opt{\bc^\epsilon}$ itself, for measures $\bmu, \bnu$ plotted in Figure~\ref{fig:expes:original_points}. We use multidimensional scaling on the adversarial cost matrix $\opt{\bc^\epsilon}$ (with distances between points from the same measures unchanged) to recover points in $\R^2$. For large values of $\epsilon$, the adversarial cost goes to $0$, which corresponds in the primal to a fully diffusive transport plan $\bpi = \bmu \bnu^\top$.
\begin{figure}[!h]
	\centering
    \captionsetup[subfigure]{justification=centering}
   	\begin{subfigure}[b]{0.15\textwidth}
	    \centering
        \includegraphics[width=\textwidth]{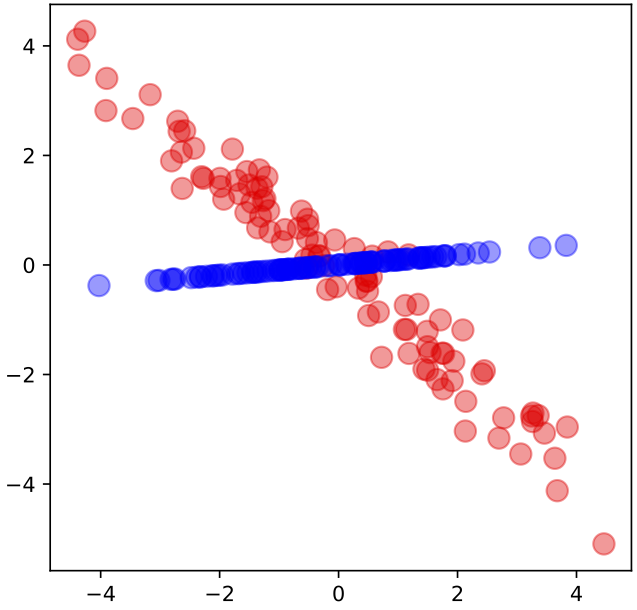}
        \vskip -0.1cm
		\caption{Original Points}\label{fig:expes:original_points}
   	 \end{subfigure}
   	\begin{subfigure}[b]{0.15\textwidth}
	    \centering
        \includegraphics[width=\textwidth]{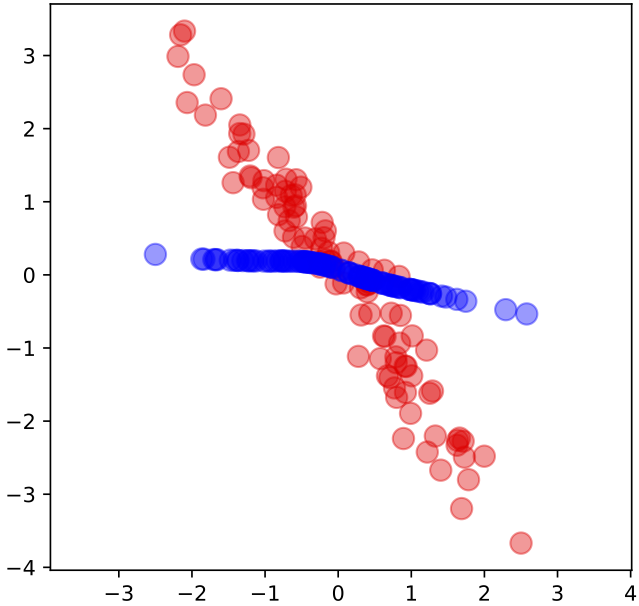}
        \vskip -0.1cm
		\caption{$\epsilon = 0.01$}
   	 \end{subfigure}
   	\begin{subfigure}[b]{0.15\textwidth}
	    \centering
        \includegraphics[width=\textwidth]{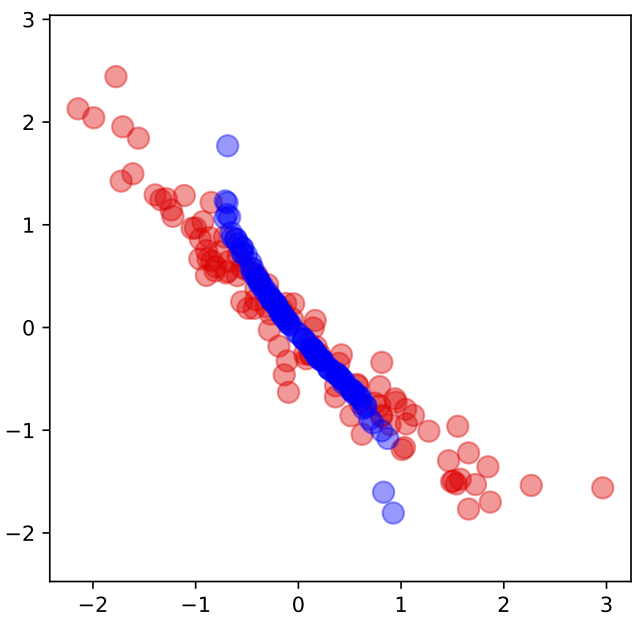}
        \vskip -0.1cm
		\caption{$\epsilon = 2$}
   	 \end{subfigure}
    \vskip -0.2cm
   	\caption{Effect of the regularization strength on the metric: as $\epsilon$ grows, the associated adversarial cost shrinks the distances.}
	\label{fig:expes:visualize_epsilon_cost}
\end{figure}

\subsection{Learning a Metric on the Color Space}

We consider 20 measures $(\bmu_i)_{i=1,\ldots,10}, (\bnu_j)_{j=1,\ldots,10}$ on the red-green-blue color space identified with $\X = [0,1]^3$. Each measure is a point cloud corresponding to the colors used in a painting, divided into two types: ten portraits by Modigliani ($\bmu_i, i \in M$) and ten by Schiele ($\bnu_j, j \in S$). As in SRW and sequential SRW formulations, we learn a metric $c_\Omega \in \CXX$ parameterized by a matrix $0 \preceq \Omega \preceq I$ such that $\trace{\Omega} = 1$ that best separates the Modiglianis and the Schieles:
\[
    \opt{\Omega} \in \argmax_{\Omega \in \mathcal{R}_1} \sum_{i \in M} \sum_{j \in S} \OTcost_{d^2_\Omega}(\bmu_i, \bnu_j).
\]
We compute $\opt{\Omega}$ using projected SGD. We then use this ``one-dimensional'' metric $d^2_{\opt{\Omega}}$ as a ground metric for OT-based color transfer~\cite{rabin2014adaptive}: an optimal transport plan $\bpi$ between two color palettes $\bmu_i, \bnu_j$ gives a way to transfer colors from one painting to the other. Visually, transferring the colors using the classical quadratic cost $\|\cdot\|^2$ or the adversarially-learnt one-dimensional metric $d^2_{\opt{\Omega}}$ makes no major difference, showing that when regularized, OT can extract sufficient information from lower dimensional representations.

\begin{figure}[!h]
	\centering
    \captionsetup[subfigure]{justification=centering}
  	\begin{subfigure}[b]{0.1\textwidth}
	    \centering
        \includegraphics[width=\textwidth]{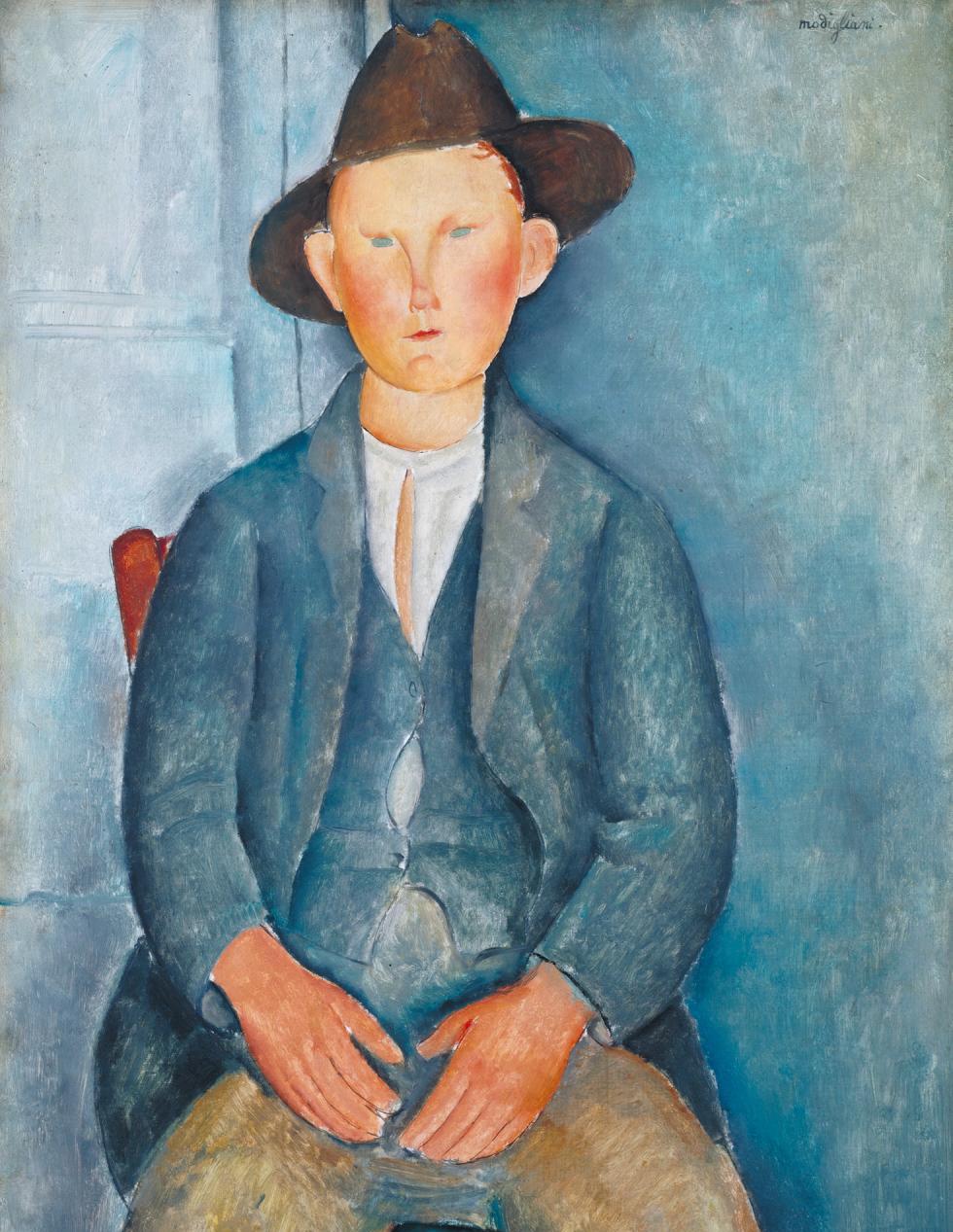}
        \vskip -0.1cm
		\caption{Modigliani}
  	 \end{subfigure}
  	\begin{subfigure}[b]{0.1\textwidth}
	    \centering
        \includegraphics[width=\textwidth]{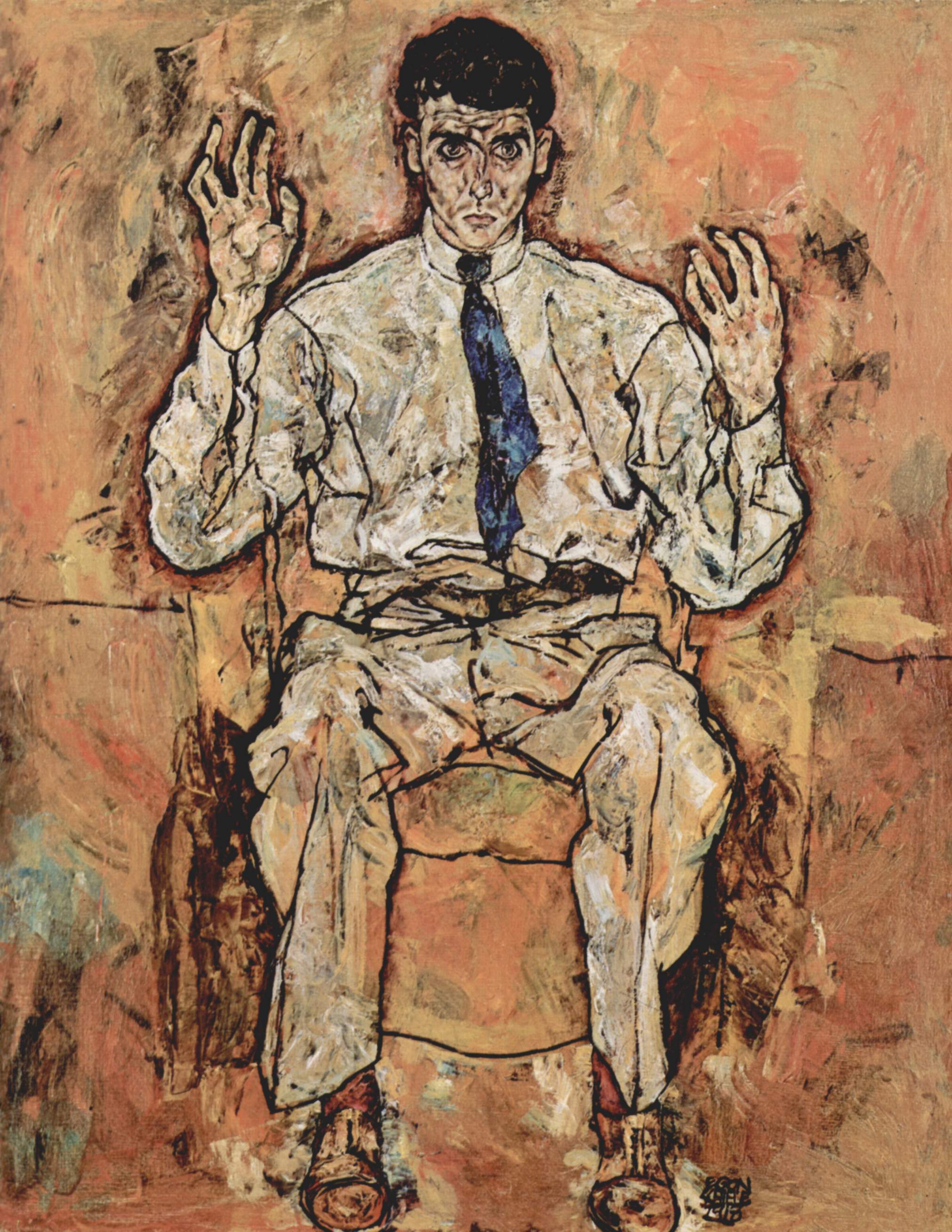}
        \vskip -0.1cm
		\caption{\phantom{aaa} Schiele}
  	 \end{subfigure}
  	\begin{subfigure}[b]{0.1\textwidth}
	    \centering
        \includegraphics[width=\textwidth]{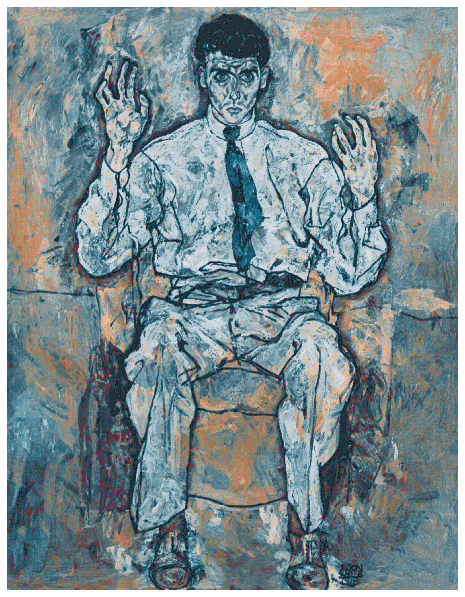}
        \vskip -0.1cm
		\caption{One-dimensional}
  	 \end{subfigure}
  	\begin{subfigure}[b]{0.1\textwidth}
	    \centering
        \includegraphics[width=\textwidth]{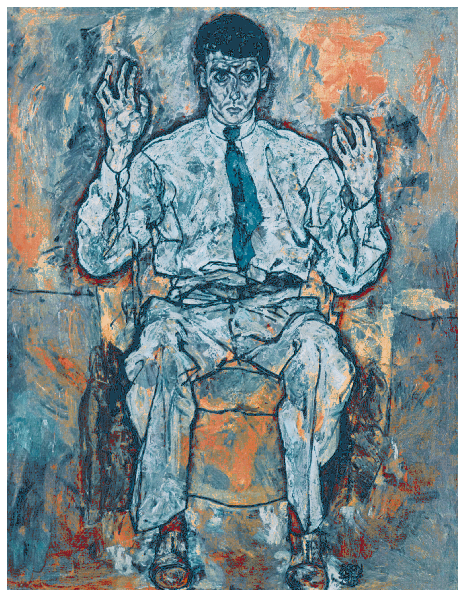}
        \vskip -0.1cm
		\caption{Three-dimensional}
  	 \end{subfigure}
    \vskip -0.3cm
  	\caption{Color transfer, best zoomed in. \textit{(a)} and \textit{(b)}: Original paintings. \textit{(c)}: Schiele's painting with Modigliani's colors, using the learn adversarial one-dimensional metric $d^2_{\opt{\Omega}}$. \textit{(d)}: Schiele's painting with Modigliani's colors, using the Euclidean metric $\|\cdot\|^2$.}
	\label{fig:expes:colortransfer}
\end{figure}
\section{Conclusion}
In this paper, we have shown that any convex regularization of optimal transport can be recast as a ground cost adversarial problem. Under some technical assumption on the regularization, we proved a duality theorem for regularized OT, which we use to characterize the optimal ground cost as a separate function of its two arguments. In order to overcome this degeneration, we proposed to constrain the robust ground-cost to take non-negative values. We also proposed a framework to learn an adversarial sequence of ground costs which is adversarial to a time-varying sequence of measures.
Future work includes learning a continuous adversarial cost $c_\theta$ parameterized by a neural network, under some regularity constraints (\textit{e.g.} $c_\theta$ is Lipschitz). On the application side, learning low-dimensional representations of time-evolving data could be applied in biology as a refinement of the methodology of~\cite{schiebinger2019optimal}.

\paragraph{Acknowledgements} We acknowledge the support of a "Chaire d’excellence de l’IDEX Paris Saclay". We would like to thank Boris Muzellec and Théo Lacombe for fruitful discussions and relevant remarks.

\newpage
\bibliography{biblio}
\bibliographystyle{icml2020}

\clearpage
\appendix

\section{Proofs}

\subsection{Proof for Proposition~\ref{adversarial:prop:c=grad(pi)}}

\begin{proof}
    Let $\opt{\pi}$ be a minimizer of~\eqref{eqn:adversarial:Wf}. Then using the optimality condition for $\sup_{c \in \CXX} \int c\,d\pi - F^*(c)$, any $c$ such that $\opt{\pi} \in \partial F^*(c)$ is a best response to $\opt{\pi}$. But by Fenchel-Young inequality, such $c$ are exactly those in $\partial F(\opt{\pi}) = \{\nabla F(\opt{\pi})\}$. Since $\nabla F(\opt{\pi})$ is the unique best response to $\opt{\pi}$, it is necessarily optimal in~\eqref{eqn:adversarial:Wf_is_adv}. Conversely, if there is a unique maximizer $\opt{c}$, then as a result of the above, $\opt{c} = \nabla F(\opt{\pi})$ for some minimizer $\opt{\pi}$ of the primal. Then $\nabla F^*(\opt{c})$ is optimal in the primal.
\end{proof}

\subsection{Proof for Remark~\ref{adv:rmk:concave}}
\begin{proof}
    As in the proof of Theorem~\ref{adversarial:thm:Wf_is_adv}:
    \begin{align*}
        \inf_{\pi \in \Pi(\mu,\nu)} F(\pi)
        &= \inf_{\pi \in \Pi(\mu,\nu)} -(-F)^{**}(\pi) \\
        &= \inf_{\pi \in \Pi(\mu,\nu)} - \sup_{c \in \CXX} \int c \,d\pi - (-F)^{*}(c) \\
        &= \inf_{\pi \in \Pi(\mu,\nu)} \inf_{c \in \CXX} \int -c \,d\pi + (-F)^{*}(c) \\
        &= \inf_{\pi \in \Pi(\mu,\nu)} \inf_{c \in \CXX} \int c \,d\pi + (-F)^{*}(-c) \\
        &= \inf_{c \in \CXX} \OTcost_c(\mu,\nu) + (-F)^{*}(-c).
    \end{align*}
\end{proof}

\subsection{Proof for Proposition~\ref{adversarial:prop:discrete_positivecost}}
\begin{proof}
    As in the proof of Theorem~\ref{adversarial:thm:Wf_is_adv}, we use Sion's minimax theorem to get
    \[
        \begin{split}
            &\adjustlimits\sup_{\bc \in \Rnn_+} \min_{\bpi \in \Pi(\bmu,\bnu)} \langle \bc, \bpi \rangle - \epsilon \sum_{ij} R_{ij}^* \left(\frac{\bc_{ij} - {\bc_0}_{ij}}{\epsilon}\right)\\
            &= \adjustlimits \min_{\bpi \in \Pi(\bmu,\bnu)} \sup_{\bc \in \Rnn_+} \langle \bc, \bpi \rangle - \epsilon \sum_{ij} R_{ij}^* \left(\frac{\bc_{ij} - {\bc_0}_{ij}}{\epsilon}\right).
        \end{split}
    \]
    Since the optimization in $\bc \in \Rnn_+$ is separable, we only need to consider this optimization coordinate by coordinate, \emph{i.e.} we only need to compute $\sup_{\bc_{ij} \in \R_+} \bpi_{ij} \bc_{ij} - f_{ij}^*(\bc_{ij})$ for all $i,j \in \range{n}$, where $f_{ij}^*(\bc_{ij}) = \epsilon R_{ij}^*\left(\frac{\bc_{ij} - {\bc_0}_{ij}}{\epsilon}\right)$.
    
    Fix $\bpi \in \Pi(\bmu,\bnu)$ and $i,j \in \range{n}$, and define $g_{ij}: \R \ni \bc_{ij} \mapsto \bpi_{ij} \bc_{ij} - f_{ij}^*(\bc_{ij})$.
    
    Suppose that $z_{ij} = f_{ij}'(\bpi_{ij}) \geq 0$. Then
    \[
        f_{ij}(\bpi_{ij}) = f_{ij}^{**}(\bpi_{ij}) = g_{ij}(z_{ij}) = \sup_{\bc_{ij} \in \R} g_{ij}(\bc_{ij}),
    \]
    and since $z_{ij} \geq 0$, $\sup_{\bc_{ij} \in \R_+} g_{ij}(\bc_{ij}) = f_{ij}(\bpi_{ij})$. This means that $\widehat R_{ij}(\bpi_{ij}) = R_{ij}(\bpi_{ij})$.
    
    Suppose now that $z_{ij} = f_{ij}'(\bpi_{ij}) < 0$. This means that
    \[
        \sup_{\bc_{ij} \in \R_+} g_{ij}(\bc_{ij}) < \sup_{\bc_{ij} \in \R} g_{ij}(\bc_{ij}).
    \]
    Since $g_{ij}$ is concave, this shows that $\sup_{\bc_{ij} \in \R_+} g_{ij}(\bc_{ij}) = g_{ij}(0) = -f_{ij}^*(0)$, \emph{i.e.} $\widehat R_{ij}(\bpi_{ij}) = \frac{-{\bc_0}_{ij}}{\epsilon} \bpi_{ij} - R_{ij}^*\left(\frac{-{\bc_0}_{ij}}{\epsilon}\right)$.
    
    Since $R_{ij}$ is convex, $R_{ij}'$ is increasing with pseudo-inverse ${R_{ij}^*}'$. Furthermore, the optimality condition in the convex conjugate problem gives, for any $\alpha \in \R$:
    \[
        R_{ij}^*(\alpha) = \alpha \times {R_{ij}^*}'(\alpha) - R_{ij} \circ {R_{ij}^*}'(\alpha).
    \]
    So if $R_{ij}$ is of class $C^1$, taking $\alpha = \frac{-{\bc_0}_{ij}}{\epsilon}$, as $x$ increases to ${R_{ij}^*}'\left( -\frac{{\bc_0}_{ij}}{\epsilon} \right)$:
    \[
        \widehat R_{ij}(x) \longrightarrow R_{ij} \circ {R_{ij}^*}'\left( -\frac{{\bc_0}_{ij}}{\epsilon} \right) = \widehat R_{ij} \circ {R_{ij}^*}'\left( -\frac{{\bc_0}_{ij}}{\epsilon} \right),
    \]
    meaning that $\widehat R_{ij}$ is of class $C^1$.
\end{proof}

\subsection{Proof for Example~\ref{adversarial:ex:pp-reg}}
\begin{proof}
    We denote by $\sign(x)$ the set $\{+1\}$ if $x>0$, $\{-1\}$ if $x<0$ and $[-1,1]$ if $x=0$.
    We apply Corollary~\ref{adversarial:thm:regularization} with $R: \Rnn \to \R$ defined as $R(\bpi) = \frac{1}{p} \|\bpi\|_{\bw, p}^p$, for which we need to compute its convex conjugate:
    \begin{align*}
        R^*(\bc) &= \sup_{\bpi \in \Rnn} \langle \bpi, \bc \rangle - \frac{1}{p} \sum_{ij} \bw_{ij} |\bpi_{ij}|^p.
    \end{align*}
    Subdifferentiating with respect to $\bpi_{ij}$:
    \begin{align*}
        \bc_{ij} &\in \frac{1}{p} \bw_{ij} \frac{\partial}{\partial \bpi_{ij}} |\bpi_{ij}|^p\\
        &= \bw_{ij} \sign(\bpi_{ij}) |\bpi_{ij}|^{p-1}
    \end{align*}
    This implies that $\sign(\bpi_{ij}) = \sign(\bc_{ij})$, so:
    \[
        \bpi_{ij} = \sign(\bc_{ij}) \left| \frac{\bc_{ij}}{\bw_{ij}} \right|^{q-1}.
    \]
    Finally,
    \begin{align*}
        R^*(\bc) &= \sum_{ij} \bc_{ij} \sign(\bc_{ij}) \left| \frac{\bc_{ij}}{\bw_{ij}} \right|^{q-1} - \frac{1}{p} \bw_{ij} \left| \frac{\bc_{ij}}{\bw_{ij}} \right|^q\\
        &= \frac{1}{q} \sum_{ij} \frac{1}{\bw_{ij}^{q-1}} |\bc_{ij}|^q\\
        &= \frac{1}{q} \|\bc\|_{{1/\bw}^{q-1},q}^q.
    \end{align*}
\end{proof}

\subsection{Proof for Example~\ref{adversarial:ex:tsallis}}
\begin{proof}
    Since $\bpi \in \Pi(\bmu,\bnu)$, $\sum_{ij} \bpi_{ij} = 1$ so we can drop it for now and only consider the term $R(\bpi) = \frac{1}{q-1} \|\bpi\|_{q}^q$ which is separable in the coordinates of $\bpi$:
    \[
        R(\bpi) = \sum_{ij} f(\bpi_{ij})
    \]
    where we have defined the convex function
    \[
        f(x) = \begin{cases}
            \frac{1}{q-1} x^q & \text{if } x \geq 0\\
            +\infty & \text{otherwise.}
        \end{cases}
    \]
    We compute its convex conjugate:
    \begin{align*}
        f^*(y) &= \sup_{x \geq 0} \left\{ xy - \frac{1}{q-1} x^q \right\}\\
        &=
        \begin{cases}
            \left(\frac{y}{p}\right)^p & \text{if } y \leq 0\\
            +\infty & \text{if } y > 0
        \end{cases}
    \end{align*}
    where $p = \frac{q}{q-1} \leq 0$ is such that $1/p + 1/q = 1$. Then $R^*(\bc) = +\infty$ if $\bc$ has a positive entry, and over $\Rnn_-$:
    \begin{align*}
        R^*(\bc) &= \sum_{ij} f^*(\bc_{ij}) = \sum_{ij} \left(\frac{\bc_{ij}}{p}\right)^p \\
        &= \sum_{ij} \left(\frac{-\bc_{ij}}{-p}\right)^p \\
        &= (-p)^{-p} \sum_{ij} \left(\frac{1}{-\bc_{ij}}\right)^{-p}.
    \end{align*}
    Adding the term $\frac{\epsilon}{1-q}$ we left aside to the result of Corollary~\ref{adversarial:thm:regularization}, we find that Tsallis regularized OT is equal to:
    \begin{align*}
        &\sup_{\bc \in \Rnn} \OTcost_{\bc}(\bmu,\bnu) - \epsilon R^*\left(\frac{\bc - \bc_0}{\epsilon}\right) + \frac{\epsilon}{1-q} \\
        &= \sup_{\substack{\bc \in \Rnn\\ \bc \leq \bc_0}} 
                \OTcost_{\bc}(\bmu,\bnu) 
                - \epsilon (-p)^{-p} \sum_{ij} \left[\frac{\epsilon}{{\bc_0}_{ij}-\bc_{ij}}\right]^{-p} \\
            &\qquad\qquad\qquad\qquad + \frac{\epsilon}{1-q} \\
        &=\sup_{\substack{\bc \in \Rnn\\ \bc \leq \bc_0}} 
                \OTcost_{\bc}(\bmu,\bnu)
                - \epsilon^{\frac{1}{1-q}} (-p)^{-p} \sum_{ij} \left[\frac{1}{{\bc_0}_{ij}-\bc_{ij}}\right]^{-p} \\
            &\qquad\qquad\qquad\qquad + \frac{\epsilon}{1-q} \\
        &= \sup_{\substack{\bc \in \Rnn\\ \bc \leq \bc_0}} 
                \OTcost_{\bc}(\bmu,\bnu)
                - \epsilon^{\frac{1}{1-q}} (-p)^{-p}  \left\| \frac{1}{\bc_0-\bc} \right\|_{-p}^{-p} \\
            &\qquad\qquad\qquad\qquad + \frac{\epsilon}{1-q}.
    \end{align*}
    
\end{proof}

\subsection{Proof for Subsection~\ref{subsec:algo:sinkhornlike}}

\paragraph{Entropic OT}
In the case of entropic OT,
\[
    F(\bpi) = \langle \bpi, \bc_0 \rangle + \epsilon \sum_{ij} \bpi_{ij} \left[ \log \bpi_{ij} - 1 \right],
\]
so
\[
    F^*(\bc) = \epsilon \sum_{ij} \exp \left( \frac{\bc_{ij} - {\bc_0}_{ij}}{\epsilon} \right)
\]
and
\[
    \nabla F^*(\bc) = \left[ \exp \left( \frac{\bc_{ij} - {\bc_0}_{ij}}{\epsilon} \right) \right]_{ij}.
\]

Then the system of equations~\eqref{eqn:algo:solvephi}~\eqref{eqn:algo:solvepsi} is:
\begin{align*}
    \forall i,\, \bmu_i &= \sum_j \exp \left( \frac{\opt{\bphi}_i + \opt{\bpsi}_j - {\bc_0}_{ij}}{\epsilon} \right) \\
    &= \exp(\opt{\bphi}_i/\epsilon) \left[K \exp(\opt{\bpsi}/\epsilon) \right]_i\\
    \forall j,\, \bnu_j &= \sum_i \exp \left( \frac{\opt{\bphi}_i + \opt{\bpsi}_j - {\bc_0}_{ij}}{\epsilon} \right) \\
    &= \exp(\opt{\bpsi}_j/\epsilon) \left[ K^\top \exp(\opt{\bphi}/\epsilon) \right]_j
\end{align*}
where $K = \exp(-{\bc_0}/\epsilon) \in \Rnn$ and $\exp$ is taken elementwise. Then solving alternatively for $\bphi$ and $\bpsi$ is exactly Sinkhorn algorithm.

\paragraph{Quadratic OT} In the case of quadratic OT, using the notations and results from example~\ref{properties:example:norm_p}:
\[
    F(\bpi) = \langle \bpi, \bc_0 \rangle + \epsilon \varphi_2(\bpi_{ij}),
\]
and
\[
    F^*(\bc) = \frac{1}{2\epsilon} \sum_{ij} \left[ \left( \bc_{ij} - {\bc_0}_{ij} \right)_+ \right]^2.
\]
Then:
\[
    \nabla F^*(\bc) = \frac{1}{\epsilon} \left( \bc - \bc_0 \right)_+.
\]
The system of equations~\eqref{eqn:algo:solvephi}~\eqref{eqn:algo:solvepsi} is:
\begin{align*}
    \forall i,\, \epsilon \bmu_i &= \sum_j \left( \opt{\bphi}_i + \opt{\bpsi}_j - {\bc_0}_{ij} \right)_+ \\
    \forall j,\, \epsilon \bnu_j &= \sum_i \left( \opt{\bphi}_i + \opt{\bpsi}_j - {\bc_0}_{ij} \right)_+
\end{align*}
which is what~\cite{blondelsmoothandsparse} solve in their appendix B.

\end{document}